\pdfoutput=1
\documentclass[11pt]{article}
\usepackage[final]{acl}

\usepackage{times}
\usepackage{latexsym}
\usepackage[T1]{fontenc}
\usepackage[utf8]{inputenc}
\usepackage{microtype}
\usepackage{inconsolata}
\usepackage{graphicx}
\usepackage{amsmath}
\usepackage{amssymb}
\usepackage{amsthm}
\usepackage{booktabs}
\usepackage{enumitem}
\usepackage{xspace}
\usepackage{multirow}
\usepackage{placeins}
\usepackage{algorithm}
\usepackage{algpseudocode}
\usepackage{fontawesome5}

\newcommand{\methodfull}{Sampling-First Structured Graph Encoding\xspace}
\newcommand{\method}{\texorpdfstring{\mbox{S\textsuperscript{2}\kern-.08em GE}}{S2GE}\xspace}

\newcommand{\digitprobe}{\texttt{digit8}\xspace}
\newcommand{\macrostate}{macro-state\xspace}

\newcommand{\appendixref}[1]{\hyperref[#1]{Appendix~\ref*{#1}}}

\theoremstyle{plain}

\newtheorem{proposition}{Proposition}

\theoremstyle{definition}

\theoremstyle{remark}

\title{%
\includegraphics[width=0.48\textwidth]{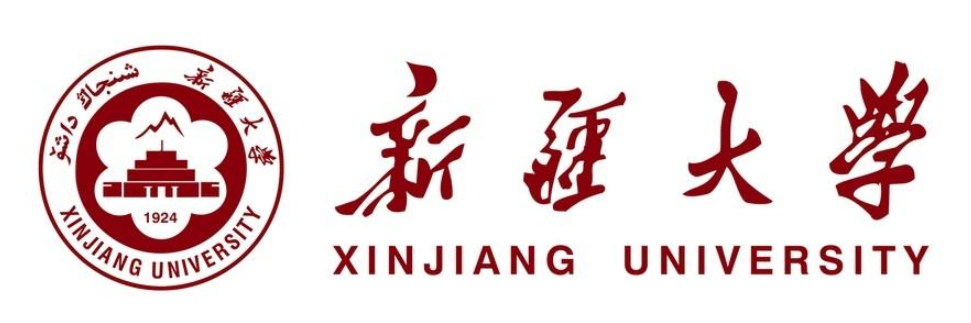}\\[0.8em]
Graph Evidence Is Not Enough: Diagnosing Native Decoder Use in Graph-Augmented LLMs%
}

\author{
\textbf{Xiaoyu Guo}\textsuperscript{1},
\textbf{Pengcheng Chen}\textsuperscript{2},
\textbf{Jiong Yu}\textsuperscript{2},
\textbf{Yi Lu}\textsuperscript{2},
\textbf{Yaohua Wang}\textsuperscript{1},
\textbf{Ziyang Li}\textsuperscript{1}\thanks{Corresponding author.}
\\
\textsuperscript{1}School of Software, Xinjiang University,
Urumqi 830002, China \\
\textsuperscript{2}School of Computer Science and Technology, Xinjiang University,
Urumqi 830046, China \\
\faEnvelope\,\href{mailto:107552501648@stu.xju.edu.cn}{\texttt{107552501648@stu.xju.edu.cn}}\quad
\textbf{Correspondence:} \href{mailto:liziyang@xju.edu.cn}{\texttt{liziyang@xju.edu.cn}} \\
\textbf{\faGithub\,Code:} \href{https://github.com/dogeee-debug/S2GE}{\texttt{github.com/dogeee-debug/S2GE}}
}

\begin{document}
\maketitle
\begin{abstract}
Graph-augmented large language models often assume that graph evidence produced
by external computation and placed in the input can be used by the native decoder. We test this assumption with
HopQA, a deliberately bounded diagnostic that asks for the shortest-hop
distance between two query nodes. Because the answer is a small integer and the
target is purely topological, failure cannot be dismissed as open-ended
generation or ambiguous evaluation. Yet existing graph-augmented baselines
still fail on this setting, showing that providing graph evidence is not the
same as making it usable. We introduce an intervention triangle with three
matched conditions: readable graph evidence, shuffled graph evidence, and
no-graph input. This separates evidence inclusion, structural readability, and
decoder-usable topology. Guided by this diagnosis, we present S$^2$GE as an
instance showing that diagnosis-driven interface design can improve native
decoder usability. S$^2$GE uses query-aware sampling, endpoint and
proximity-based ordering, and structure-preserving alignment. Across DBLP,
Biomedical, GoodReads, and PubMed, S$^2$GE achieves
strict exact-match scores of $36.5\%$, $57.8\%$, $76.6\%$, and $52.0\%$,
improving over the strongest native-generation baseline by $53.5$ points on average. The
interventions further reveal harmful-shuffle, shuffle-robust, and
no-graph-saturated regimes.

\end{abstract}

\section{Introduction}
\label{sec:intro}
\begin{figure}[!b]
    \centering
    \includegraphics[width=0.92\columnwidth]{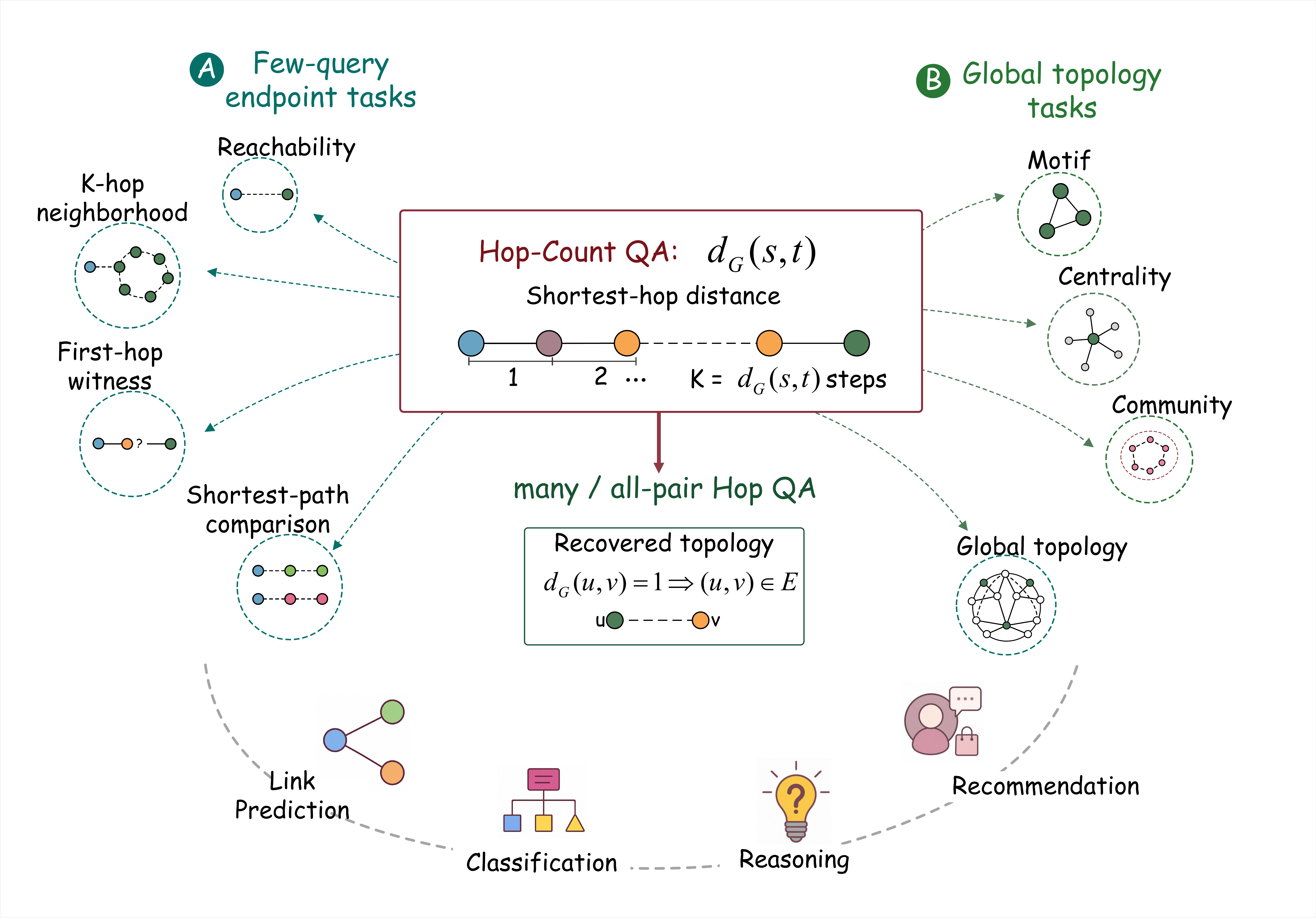}
    \caption{HopQA as a bounded topology diagnostic for testing graph-evidence use on the native generation surface.}
    \label{fig:hopqa-topology-diagnostic}
\end{figure}

External execution is the right abstraction for deterministic graph
computation: shortest paths, reachability, and neighborhood queries are better
handled by graph algorithms than by free-form generation. Even if future
graph-language systems delegate such computation to external executors, an
interface problem remains: task-relevant graph structure must be translated
into representations that the model can read and use. Sampling, serialization,
projection, and role ambiguity can distort this translation before the decoder
acts. We use HopQA as a bounded probe of this translation fidelity, measuring
whether externally produced graph evidence survives the graph-to-decoder
interface well enough to support an exact native answer.

Graph-augmented LLMs expose graph structure to language models before
generation. Existing surveys describe common interface forms, including graph
as text, graph tokens, graph-enhanced generation, and LLM-only prompting
\citep{ren2024survey,li2024graphmeets}. Representative systems
follow this dependency through different interfaces: G-Retriever retrieves graph
evidence, and LLaGA maps graph structure into tokens for the decoder
\citep{he2024gretriever,chen2024llaga}. In each case, the decoder must turn
exposed evidence into an answer.

To validate whether such evidence becomes usable, Hop-count question answering
(HopQA) tests the graph-to-decoder path in a compact form. Given a source node
and a target node, the model must generate their shortest hop distance from
exposed graph evidence. We evaluate this with strict exact match, an existing
standard under which current baselines fail. HopQA therefore asks whether
exposed graph evidence can become an exact generated answer through ordinary
decoding. Figure~\ref{fig:hopqa-topology-diagnostic} summarizes this
evidence-to-answer workflow.

Under this test, existing graph-augmented LLMs fail sharply on the
\textit{Core-HopQA} family: DBLP and PubMed cover citation-style graphs,
Biomedical covers a dense biomedical relation graph, and GoodReads covers a
recommendation-style graph. G-Retriever obtains $0.0{\pm}0.0\%$ strict EM across
these four domains, and LLaGA remains near zero. Graph evidence reaches the
decoder, while exact hop answers remain absent from the native output. On the
same examples, graph-only classifiers and SubgraphRAG retrieval-execution
checks extract hop signal from graph structure and retrieved evidence graphs.
The second evaluation family, \textit{Auxiliary Graph Diagnostics}, then probes
path witnesses, graph-token interventions, and adjacency transfer.

This gap points to three interface requirements: relevant structure must enter
the bounded evidence budget, endpoint roles must remain readable, and projection
should preserve local adjacency. We propose \methodfull{} (\method) with
query-aware sampling, role-based perception, and adjacency-based
alignment. Sampling comes first to expose endpoint-conditioned evidence; readable,
absent, and shuffled graph conditions then test how evidence organization
changes the decoder's answer state.

The paper makes three contributions.
\begin{itemize}[leftmargin=1.2em]
    \item \textbf{A diagnostic gap between graph evidence and native decoder use.}
          We identify a gap between graph evidence reaching the decoder and exact hop answers appearing in native generation. On HopQA, trained graph-augmented LLM baselines fail under strict exact match even though the answer space is bounded and graph signal can be extracted.

    \item \textbf{A diagnostic design for separating graph signal from decoder use.}
          We design a diagnostic protocol that separates signal existence, evidence exposure, retrieval-execution recovery, and native generation. The readable/no-graph/shuffled intervention triangle is the graph-token tool used inside this protocol.

    \item \textbf{Residual graph signal in generated outputs.}
          We expose measurable residual effects of graph-token changes on generated answers, allowing us to distinguish helpful residue, harmful residue, and output collapse across data regimes.
\end{itemize}

\section{Related Work}
\label{sec:related}
\paragraph{From graph signal to visible evidence.}
Graph-language systems make structure visible through retrieval, text
serialization, and learned graph-token interfaces. G-Retriever retrieves compact
textual subgraphs before generation \citep{he2024gretriever}; LLaGA projects
graph representations into the LLM input space \citep{chen2024llaga}. GraphRAG
and GRAG organize retrieved evidence around graph structure
\citep{edge2025graphrag,hu2025grag}, while KG$^2$RAG and SubgraphRAG emphasize
knowledge-graph guidance or adjustable subgraph retrieval
\citep{zhu2025kg2rag,li2025subgraphrag}. A recent diagnostic anchor is
\citet{zhou2026whatbreaks}, who benchmark KG-RAG under incomplete knowledge and show
that retrieval-centered systems can struggle when reasoning must go beyond
explicit triples. Recent GraphRAG surveys further organize the pipeline around
query processing, retrieval, organization, and generation
\citep{han2025graphrag,zhang2025graphragsurvey}. GraphRAG-Bench asks when graph
structure helps RAG \citep{xiang2026whengraphs}, while ROGRAG and GFM-RAG
further introduce graph-specific RAG or retrieval frameworks
\citep{wang2025rograg,luo2026gfmrag}.
These works make graph evidence more visible and controllable, while
graph-signal existence, exposed evidence, and decoder usability remain only
partly separated.

\paragraph{From visible evidence to readable interfaces.}
Visible evidence becomes useful only when its organization is readable to the
decoder. GraCoRe and GraphOmni provide benchmark-level evidence that graph
reasoning performance is sensitive to task format, prompt design, and
evaluation coverage \citep{yuan2025gracore,xu2026graphomni}. GraphArena,
GraphEval36K, and recent
graph-generation evaluations extend the same concern to algorithmic graph
computation, code-style graph tasks, and structural generation
\citep{tang2025grapharena,wu2025grapheval36k,demirci2025graphsavvy}.
For learned interfaces, \citet{hoyle-etal-2021-promoting} show that graph-to-text
models are sensitive to alternative graph linearizations. \citet{zhang2026revisiting}
further revisit graph-tokenized language models and show that compression alone
does not guarantee graph understanding. The most
direct readability anchor is \citet{chaudhary2026graff}: GRAFF uses
graph-augmented fine-grained fusion to preserve node-level and relation-level
cues. Recent graph-token and graph-foundation-model work explores discrete graph
tokenization and reconstructive graph instruction tuning
\citep{guo2026graphtokenization,zhang2026rglm,xiang2026quantized}. Graph
foundation models and KG foundation models extend the same interface problem to
graph-language pretraining and semantic KG transfer
\citep{wang2025gfmsurvey,arun2025semma,kong2025gofa}.
These systems improve graph exposure and evaluation breadth, while final scores
can still hide whether errors come from missing evidence, unreadable
organization, or the generation surface.

\paragraph{From readable evidence to native usability.}
Many graph reasoning systems strengthen task performance by adding external
execution, symbolic traversal, tool calls, candidate scoring, constrained
decoding, or supervised reasoning traces. Graph-CoT and KG-CoT add reasoning
traces over graphs or knowledge graphs \citep{jin2024graphcot,zhao2024kgcot}.
KiRAG and HopRAG strengthen retrieval-reasoning pipelines for knowledge-driven
or logic-aware QA \citep{fang2025kirag,liu2025hoprag}. MIAoG uses multi-view
adaptive reasoning for KG-enhanced LLMs \citep{zhang2026miag}, and GLOW combines
GNN candidate prediction with LLM prompting for open-world KGQA
\citep{abdallah2026glow}.
GraphPile and G1 show that graph problems can also be used as training or
reinforcement-learning data for broader graph reasoning
\citep{zhang2025graphproblems,guo2025g1}.
These approaches are effective for graph tasks, while their final answers may
come from execution or scoring surfaces different from ordinary decoding.
Bounded structural queries therefore isolate native usability: when graph signal
is visible and organized, can the decoder itself produce a legal and correct
structural answer? Output-state audits connect this question to
language-generation degeneration studies
\citep{holtzman2020curious,welleck2020neural}.

\section{Bounded Graph-Signal Diagnosis}
\label{sec:method}
\begin{figure*}[t]
\centering
\includegraphics[width=\textwidth]{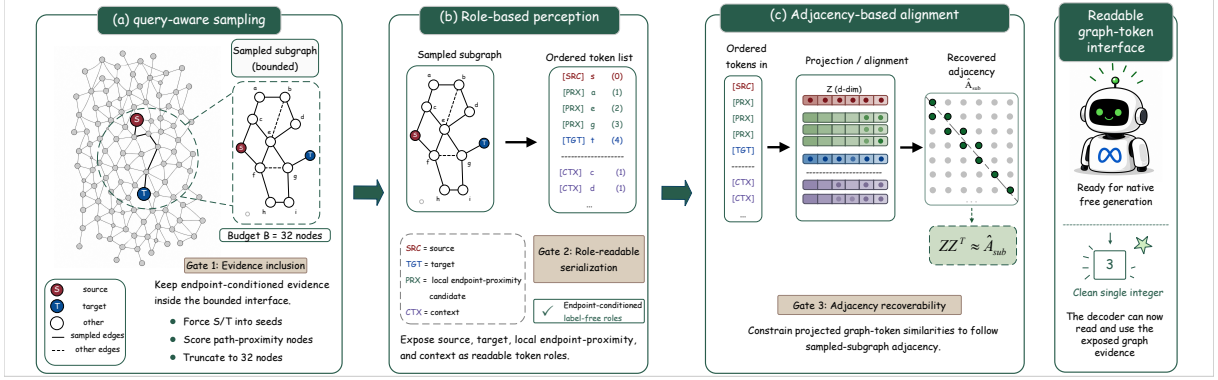}
\caption{Workflow of the \method sampling-first graph-token interface. The data flow follows three linked interface conditions: (a) query-aware sampling selects bounded endpoint-conditioned evidence, (b) role-based perception orders and annotates sampled nodes so endpoint roles and local proximity are readable, and (c) adjacency-based alignment regularizes projected tokens so sampled adjacency remains visible after projection.}
\label{fig:s2ge-workflow}
\label{fig:graph-token-interface-workflow}
\end{figure*}

\begin{table*}[t]
    \centering
    \footnotesize
    \setlength{\tabcolsep}{2.2pt}
    \renewcommand{\arraystretch}{1.04}
    \begin{tabular*}{\textwidth}{@{\hspace{2pt}\extracolsep{\fill}}lllccccc@{\hspace{2pt}}}
        \toprule
        Method & Decision surface & Domain & StrictEM $\uparrow$ & ParsedEM $\uparrow$ & $\Delta$Chance $\uparrow$ & SingleInt $\uparrow$ & Dom.Ans. $\downarrow$ \\
        \midrule
        Random & uniform legal guess & all & $20.0$ & $20.0$ & $0.0$ & $100.0$ & $20.0$ \\
        Majority & fixed majority label & all & $20.0$ & $20.0$ & $0.0$ & $100.0$ & $100.0$ \\
        \midrule
        G-Retriever & native generation & DBLP & $0.0{\pm}0.0$ & $23.2{\pm}4.4$ & $-20.0{\pm}0.0$ & $0.0{\pm}0.0$ & $38.0{\pm}4.5$ \\
        G-Retriever & native generation & Biomedical & $0.0{\pm}0.0$ & $45.8{\pm}4.4$ & $-20.0{\pm}0.0$ & $0.0{\pm}0.0$ & $31.3{\pm}6.8$ \\
        G-Retriever & native generation & GoodReads & $0.0{\pm}0.0$ & $41.3{\pm}2.6$ & $-20.0{\pm}0.0$ & $0.0{\pm}0.0$ & $34.8{\pm}4.4$ \\
        G-Retriever & native generation & PubMed & $0.0{\pm}0.0$ & $20.3{\pm}0.8$ & $-20.0{\pm}0.0$ & $0.0{\pm}0.0$ & $31.6{\pm}2.5$ \\
        \midrule
        LLaGA & native generation & DBLP & $2.3{\pm}2.1$ & $12.6{\pm}1.7$ & $-17.7{\pm}2.1$ & $10.5{\pm}8.7$ & $29.9{\pm}3.1$ \\
        LLaGA & native generation & Biomedical & $0.0{\pm}0.0$ & $20.8{\pm}2.0$ & $-20.0{\pm}0.0$ & $0.0{\pm}0.0$ & $45.6{\pm}16.9$ \\
        LLaGA & native generation & GoodReads & $6.8{\pm}5.5$ & $16.8{\pm}2.8$ & $-13.2{\pm}5.5$ & $35.5{\pm}26.6$ & $40.3{\pm}4.1$ \\
        LLaGA & native generation & PubMed & $0.0{\pm}0.0$ & $6.8{\pm}5.6$ & $-20.0{\pm}0.0$ & $0.0{\pm}0.0$ & $27.7{\pm}3.3$ \\
        \midrule
        Pure GNN & graph-only classifier & DBLP & $32.1$ & $32.1$ & $+12.1$ & $100.0$ & -- \\
        Pure GNN & graph-only classifier & Biomedical & $27.8$ & $27.8$ & $+7.8$ & $100.0$ & -- \\
        Pure GNN & graph-only classifier & GoodReads & $42.5$ & $42.5$ & $+22.5$ & $100.0$ & -- \\
        Pure GNN & graph-only classifier & PubMed & $45.3$ & $45.3$ & $+25.3$ & $100.0$ & -- \\
        \midrule
        SubgraphRAG & retrieval execution & DBLP & $24.9$ & $24.9$ & $+4.9$ & $100.0$ & -- \\
        SubgraphRAG & retrieval execution & Biomedical & $20.5$ & $20.5$ & $+0.5$ & $100.0$ & -- \\
        SubgraphRAG & retrieval execution & GoodReads & $58.2$ & $58.2$ & $+38.2$ & $100.0$ & -- \\
        SubgraphRAG & retrieval execution & PubMed & $50.5$ & $50.5$ & $+30.5$ & $100.0$ & -- \\
        \midrule
        \method & native generation & DBLP & $\mathbf{36.5{\pm}1.6}$ & $36.5{\pm}1.6$ & $+16.5{\pm}1.6$ & $100.0{\pm}0.0$ & $27.4{\pm}2.9$ \\
        \method & native generation & Biomedical & $\mathbf{57.8{\pm}0.9}$ & $57.8{\pm}0.9$ & $+37.8{\pm}0.9$ & $100.0{\pm}0.0$ & $28.4{\pm}0.3$ \\
        \method & native generation & GoodReads & $\mathbf{76.6{\pm}0.9}$ & $76.6{\pm}0.9$ & $+56.6{\pm}0.9$ & $100.0{\pm}0.0$ & $24.0{\pm}1.2$ \\
        \method & native generation & PubMed & $\mathbf{52.0{\pm}3.4}$ & $52.0{\pm}3.4$ & $+32.0{\pm}3.4$ & $100.0{\pm}0.0$ & $28.7{\pm}1.9$ \\
        \bottomrule
    \end{tabular*}
    \caption{HopQA main results. Values are percentages. G-Retriever, LLaGA, and \method use ordinary generation; Pure GNN and SubgraphRAG are numeric graph-control paths, so ParsedEM equals EM and SingleInt is $100\%$. Dom.Ans. is omitted for these controls because their sampled output distributions are unavailable. $\Delta$Chance is StrictEM minus the $20\%$ random/majority rate. Additional controls and path-found rates are in \appendixref{app:experiment-details}, Table~\ref{tab:baseline-comparison}.}
    \label{tab:native-generation}
\end{table*}

HopQA serves as a bounded diagnostic of structural-evidence usability across
the graph-to-decoder interface and covers only a controlled slice of graph
reasoning. The task is bounded in its answer space, exposed evidence budget, interface
content, and native generation surface. These restrictions separate graph-signal
existence, evidence exposure, interface readability, and decoder use. Each
example contains a graph $G=(V,E,X)$, a query $q=(s,t)$, and a gold answer
$y=d_G(s,t)\in\mathcal A$, where $\mathcal A=\{1,2,3,4,5\}$. The model receives
a bounded graph interface built from a sampled subgraph and must generate the
exact hop label through ordinary decoding. \appendixref{app:notation} provides a
notation reference table for the full paper.

For a query $q=(s,t)$, a sampler with node budget $B$ produces a sampled
subgraph
\[
G_{\mathrm{sub}}=S_B(G,q).
\]
The graph interface exposed to the language model is
\[
I_B(G,q)=I(G_{\mathrm{sub}},q)=(T(q),M(q)),
\]
where $T(q)=(z_{\pi(1)},\ldots,z_{\pi(m)})$ is the ordered sequence of projected
graph tokens, and $M(q)$ contains input-computable annotations such as endpoint
roles, local degree, and ordering keys. The decoder succeeds when it generates
the exact hop answer from this interface. We write
\[
U_\theta(q;I_B)=P_\theta(\hat y=d_G(s,t)\mid I_B(G,q)).
\]
This quantity states the diagnostic event: exposed graph evidence becomes useful
when ordinary generation produces the correct hop label.

\appendixref{app:formal-boundary} gives the formal statements behind the
diagnostic. For simple loopless unweighted graphs, all-pair hop counts determine
adjacency because an edge exists exactly when the hop distance is one.
Five-choice HopQA is a bounded slice of this distance query. In the main text,
this fact motivates HopQA as a compact graph-structure test.

\paragraph{Hypothesis 1: graph signal exists before decoder use.}
The first question is how much hop signal can be extracted from graph evidence
outside native generation. We compare ordinary generation with graph-based
controls that use graph structure directly. Let $\mathcal{C}_{\mathrm{graph}}$
be the set of graph controls, such as a GNN hop classifier and SubgraphRAG
retrieval-execution checks over retrieved evidence. For method $m$, let
$m_{\mathrm{nat}}$ denote its ordinary generation path. We define the positive
signal-use gap as
\[
\Delta_{\mathrm{su}}^{+}(m)
=
\max\Bigl(0,
\max_{c\in \mathcal{C}_{\mathrm{graph}}} \mathrm{EM}(c)
-
\mathrm{EM}(m_{\mathrm{nat}})
\Bigr).
\]
A positive $\Delta_{\mathrm{su}}^{+}(m)$ marks hop signal extracted by graph
controls beyond the ordinary generation path. This separates graph-signal
recovery from decoder use.

\paragraph{Model scale and interface boundaries.}
A larger backbone may improve optimization, instruction following, or robustness to
imperfect serialization. However, scaling acts only after the exposed interface
has been constructed. For any decoder family that observes an example solely
through $I_B(G,q)$, distinctions removed by sampling or representation cannot be
reconstructed from decoder capacity alone. Proposition~2 formalizes this
interface-conditioned upper bound, and Corollary~2.1 shows that it holds across
decoder scales. We therefore distinguish \textbf{scale-sensitive decoder
failure} from \textbf{interface-induced information loss}. Our single-backbone
experiments characterize the former only for LLaMA-3-8B-Instruct, whereas the
latter motivates the interface principles studied here independently of
backbone size.

\paragraph{Hypothesis 2: decoder use depends on readable interface conditions.}
After graph signal is established, evidence organization becomes the next
question. We describe the interface by three measurable conditions:
\[
\Phi(I_B,q)=(C_{\mathrm{local}}(q),S_{\mathrm{role}}(q),A_{\mathrm{adj}}(q)).
\]
Here $C_{\mathrm{local}}(q)$ measures whether evidence related to the query
endpoints enters the bounded sample. In experiments, it is audited by endpoint
coverage and path recall. $S_{\mathrm{role}}(q)$ measures whether the decoder
can distinguish the source, the target, nearby nodes, and context nodes. It is
tested by role and query ablations. $A_{\mathrm{adj}}(q)$ measures whether
projected graph tokens retain local adjacency in the sampled graph. It is
tested by frozen adjacency probes.

These conditions are ordered: sampling controls what can be seen, role
perception makes sampled nodes readable, and alignment preserves sampled
adjacency. Removing a limiting condition should therefore degrade native
generation.

\paragraph{Hypothesis 3: interface changes shift the generated answer distribution.}
In error-prone domains, unreadable interfaces can expose distinct answer-state
regions. For each generated answer $\hat y_i$, we record diagnostic labels such
as correct legal answer, wrong legal answer, illegal answer, long-digit
continuation, or malformed output. These audits need not form a mutually
exclusive partition. For a chosen categorical mapping
$O:\hat{\mathcal Y}\to\mathcal O$ and interface condition $I$, define
\[
p_\theta(o\mid I)
=
\frac{1}{N}\sum_{i=1}^{N}\mathbf{1}\{O(\hat y_i)=o\}.
\]
We also track dominant-answer concentration,
\[
D(I)=
\max_{a\in\{1,\ldots,5\}}
\frac{1}{N}\sum_{i=1}^{N}
\mathbf{1}\{\mathrm{extract}(\hat y_i)=a\}.
\]
Readable, no-graph, and shuffled interfaces are denoted by $I_R$, $I_N$, and
$I_S$. These diagnostics separate helpful residue, harmful residue, and stable
failure modes in generated answers.

These three hypotheses define the bounded-HopQA diagnostic: graph signal can be
extracted before decoder use, exposed evidence must remain readable, and
interface changes can shift bounded output states. The proof sketch and formal
propositions are given in \appendixref{app:formal-boundary}; Section~\ref{sec:exp}
tests the hypotheses empirically.

\subsection[\method Interface]{\method: Sampling-First Structured Graph Encoding}
\label{sec:s2ge-interface}
We propose \methodfull{} (\method). The method follows the order of the three
interface conditions: include query-related evidence, make sampled nodes
readable, and preserve local adjacency after projection. Figure~\ref{fig:s2ge-workflow}
gives the workflow.

\paragraph{Query-aware sampling.}
The sampler first places the query endpoints into the sampled evidence when
they are available. It then expands locally around endpoint-relevant and
high-degree nodes. The output is a bounded subgraph $G_{\mathrm{sub}}$. This
step targets $C_{\mathrm{local}}(q)$. It decides which part of the graph can be
seen by the decoder.

\paragraph{Role-based perception.}
Following topology-aware graph prompting and tokenization methods
\citep{fatemi2024talk,chen2024llaga,zhang2026revisiting},
\method attaches query-conditioned roles to sampled nodes before projection
reaches the decoder. Like these topology-aware interfaces, the role channel
packages endpoint identity, local proximity, and neighborhood salience as
readable token features. The ordering function uses source and target roles,
local proximity to the endpoints, degree, and a stable traversal index:
\[
\pi
=
\mathrm{sort}\bigl(
V_B;
r(v),d_s(v),d_t(v),-\deg(v),b(v)
\bigr).
\]
Here $r(v)$ gives endpoint and context roles, $d_s(v)$ and $d_t(v)$ are local
distances inside the sampled subgraph, and $b(v)$ is a stable traversal index.
This step targets $S_{\mathrm{role}}(q)$ by presenting the source, target,
nearby nodes, and context nodes as distinguishable positions in the token
sequence.

\paragraph{Adjacency-based alignment.}
Projection can weaken local edges in the sampled graph. To preserve this local
structure, \method adds an alignment loss on the projected graph tokens. Let
$Z\in\mathbb{R}^{m\times d}$ be the matrix of projected node tokens. Let $A_B$
be the sampled adjacency matrix, let $\tilde A_B=A_B+I$, and let
\[
\hat A_B=\tilde D^{-1/2}\tilde A_B\tilde D^{-1/2}
\]
be the normalized adjacency target. The alignment loss is
\[
\mathcal{L}_{\mathrm{align}}
=
\left\|
\mathrm{norm}(Z)\mathrm{norm}(Z)^\top
-
\hat A_B
\right\|_F^2,
\]
where $\mathrm{norm}(\cdot)$ applies row-wise $\ell_2$ normalization. This
step targets $A_{\mathrm{adj}}(q)$. It encourages adjacent sampled nodes to
remain close in projected token space.

\subsection{Training Objective and Diagnostics}
\label{sec:training-objective-diagnostics}
Following the interface coordinates
$\Phi(I_B,q)=(C_{\mathrm{local}}(q),S_{\mathrm{role}}(q),A_{\mathrm{adj}}(q))$,
we optimize native answer generation together with sampled-adjacency
preservation. The training objective is
\[
\mathcal{L}
=
\mathcal{L}_{\mathrm{LM}}
+\lambda \mathcal{L}_{\mathrm{align}},
\]
where $\lambda=0.25$ is fixed by validation and
\[
\mathcal{L}_{\mathrm{LM}}
=
-\sum_{j=1}^{|\mathbf y|}
\log p_\theta(y_j\mid y_{<j},I_B(G,q))
\]
is the standard token-level negative log-likelihood of the target answer
sequence. Implementation details are given in
\appendixref{app:experiment-details}.

\section{Experiments}
\label{sec:exp}
We evaluate two named families: \textit{Core-HopQA} domains for native hop-label
generation, and \textit{Auxiliary Graph Diagnostics} for signal extraction, path
evidence, intervention behavior, and adjacency transfer. Together, they test
whether exposed graph evidence becomes readable and usable under native
generation.

\subsection{Experimental Setup}

\textbf{Datasets.} Core-HopQA is constructed from public graph sources:
GRBench domains \citep{jin2024graphcot} for DBLP, Biomedical, and GoodReads,
and Planetoid/PubMed \citep{yang2016planetoid} for PubMed. The domains include
citation-style graphs (DBLP and PubMed), a biomedical relation graph
(Biomedical), and a recommendation-style graph (GoodReads).
These categories test whether topology evidence remains usable across
bibliographic, biomedical, and recommendation settings. Dataset sources, graph
statistics, split construction, and chance baselines are reported in
\appendixref{app:datasets-metrics}, Table~\ref{tab:dataset-overview}.
All domains use a balanced five-label protocol, giving random and majority
prediction a $20\%$ expected accuracy.

\textbf{Baselines.} We group baselines by decision surface. G-Retriever and
LLaGA are native-generation baselines
\citep{he2024gretriever,chen2024llaga}. Pure GNN is a graph-only control.
SubgraphRAG is the retrieval-execution control \citep{li2025subgraphrag}.
Legal-choice, raw-text, and Graph-CoT controls are reported in
\appendixref{app:experiment-details}, Table~\ref{tab:baseline-comparison}.
All trainable \method{} runs use the same maximum 12-epoch schedule with
early-stopping patience 6; G-Retriever and LLaGA are evaluated using their
released implementations and checkpoints, including method-specific
optimization settings.

\textbf{Metrics.} StrictEM is the primary native-usability metric:
\[
\mathrm{StrictEM}
=
\frac{1}{n}\sum_i \mathbf 1\{\operatorname{exact}(\hat y_i)=y_i\}.
\]
ParsedEM audits first-integer traces,
\[
\mathrm{ParsedEM}
=
\frac{1}{n}\sum_i \mathbf 1\{\operatorname{extract}(\hat y_i)=y_i\},
\]
and the positive signal-use gap is
\[
\Delta_{\mathrm{su}}^{+}(m)
=
\max\Bigl(0,\max_{c\in\mathcal C_{\mathrm{graph}}}\mathrm{EM}(c)
-\mathrm{EM}(m_{\mathrm{nat}})\Bigr).
\]
Table~\ref{tab:native-generation} also reports SingleIntRate and
DominantAnswerRate as output-state audits; full metric definitions are in
\appendixref{app:datasets-metrics}, Table~\ref{tab:metric-definitions}.

\subsection{Native Decoder Usability on Core-HopQA}

To verify native decoder usability, we evaluate these metrics on Core-HopQA.

Table~\ref{tab:native-generation} shows that G-Retriever obtains
$0.0{\pm}0.0\%$ strict EM across Core-HopQA, while LLaGA stays near zero.
\method reaches $36.5$--$76.6\%$ strict EM, corresponding to $\Delta$Chance
gains of $+16.5$, $+37.8$, $+56.6$, and $+32.0$ on DBLP, Biomedical,
GoodReads, and PubMed. Its DominantAnswerRate remains close to a balanced
answer surface, unlike majority fallback. Pure GNN and SubgraphRAG provide
graph-control EM on graph-only and retrieval-execution decision surfaces;
SubgraphRAG path-found rates are reported in \appendixref{app:experiment-details},
Table~\ref{tab:baseline-comparison}.

ParsedEM is useful mainly as an audit. Tables~\ref{tab:baseline-error-audit}
and~\ref{tab:baseline-comparison} show that first-integer extraction and
legal-choice scoring can expose residual numeric traces, while Table~\ref{tab:native-generation}
attributes the native failure to graph-evidence use and output-surface
formation together.

\begin{figure*}[t!]
    \centering
    \makebox[\textwidth][c]{\includegraphics[width=1.04\textwidth]{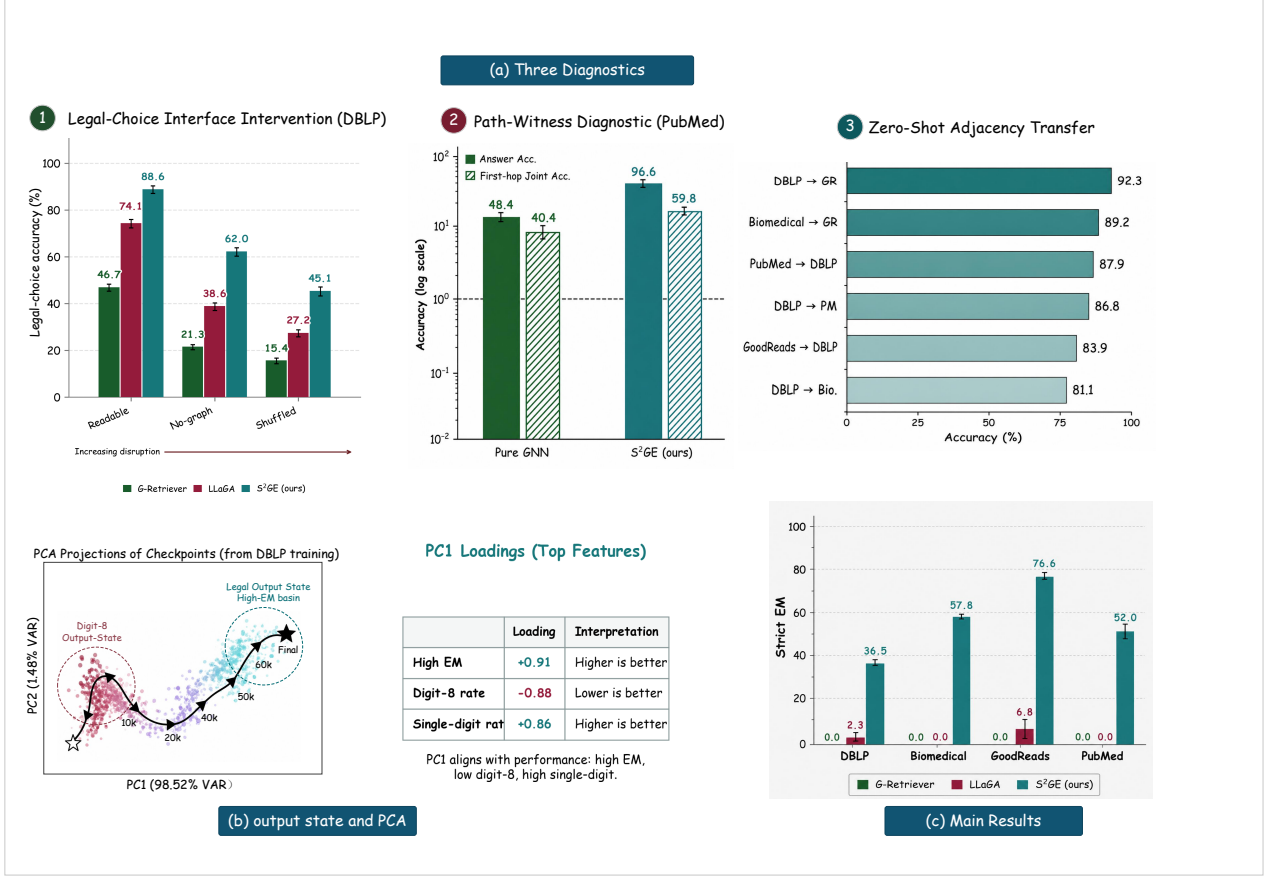}}
    \caption{Experimental and diagnostic summary for the three diagnostic hypotheses. Panel (a) combines three independently conducted diagnostics: a DBLP legal-choice interface intervention, PubMed path-witness generation, and six-direction zero-shot adjacency transfer. Panel (b) reports checkpoint-level output-state dynamics for DBLP, a harmful-shuffle domain, together with checkpoint PCA, where PC1 and PC2 explain $98.52\%$ and $1.48\%$ of the observed variance, respectively. Panel (c) reports the main native-generation StrictEM results across the four HopQA domains.}
    \label{fig:expr-summary}
\end{figure*}

Table~\ref{tab:path-witness} reports the PubMed path-witness diagnostic: beyond
the scalar hop label, the model must provide a locally valid first-hop witness.

\begin{table}[t]
    \centering
    \scriptsize
    \setlength{\tabcolsep}{3pt}
    \begin{tabular}{@{}p{0.39\columnwidth}cc@{}}
        \toprule
        Method & Answer Acc. & First-hop Joint \\
        \midrule
        Pure GNN path-witness control
               & $48.4{\pm}0.8$
               & $40.4{\pm}2.3$ \\
        \method{} path-witness
               & $\mathbf{96.6{\pm}1.4}$
               & $\mathbf{59.8{\pm}0.4}$ \\
        \bottomrule
    \end{tabular}
    \caption{PubMed path-witness diagnostic. Values are percentages. Answer accuracy measures the path-existence answer. First-hop joint accuracy requires both the correct answer and a valid first-hop witness.}
    \label{tab:path-witness}
\end{table}

Legal-choice scoring is reported separately in \appendixref{app:experiment-details},
Table~\ref{tab:baseline-comparison}.
It provides G-Retriever, LLaGA, and \method with a favorable label-selection path over $\mathcal A$.
The remaining gap suggests that native-generation failure involves both
graph-evidence use and output form.

\subsection{Graph-Token Intervention Diagnostics}

This subsection tests Hypothesis 3: how graph-token changes move generation across output
states. For domain $\kappa$, let
\[
    \begin{aligned}
        R_\kappa & =\mathrm{EM}(I_R;\kappa), \\
        N_\kappa & =\mathrm{EM}(I_N;\kappa), \\
        S_\kappa & =\mathrm{EM}(I_S;\kappa),
    \end{aligned}
\]
where $I_R$, $I_N$, and $I_S$ denote readable, no-graph, and shuffled graph
interfaces, respectively. We compare
\[
    G^{\mathrm{shuf}}_\kappa=S_\kappa-N_\kappa,
    \qquad
    L^{\mathrm{org}}_\kappa=R_\kappa-S_\kappa.
\]
Here $G^{\mathrm{shuf}}_\kappa$ is shuffled residue and
$L^{\mathrm{org}}_\kappa$ is readable organization gain. DBLP has
$S_\kappa<N_\kappa$, while PubMed has $S_\kappa$ close to $R_\kappa$; we
therefore use DBLP and PubMed as representative ablation domains.

Table~\ref{tab:token-interventions} summarizes the shared graph-token
interventions. Readable tokens give the strongest ordinary generation results.
No-graph tokens remove evidence, and shuffled tokens keep token content while
breaking order. DBLP is harmful-shuffle, and PubMed is shuffle-robust, showing
that organization affects output states as well as evidence amount.
Figure~\ref{fig:domain-interventions} visualizes the intervention triangle.
Domain-specific markers such as DBLP \digitprobe and the checkpoint-PCA
variables are defined in \appendixref{sec:diagnostics},
Section~\ref{app:diagnostic-output-vars}.

\begin{figure}[!htbp]
    \centering
    \vspace{-8pt}
    \includegraphics[width=0.9\columnwidth,trim=600 0 0 0,clip]{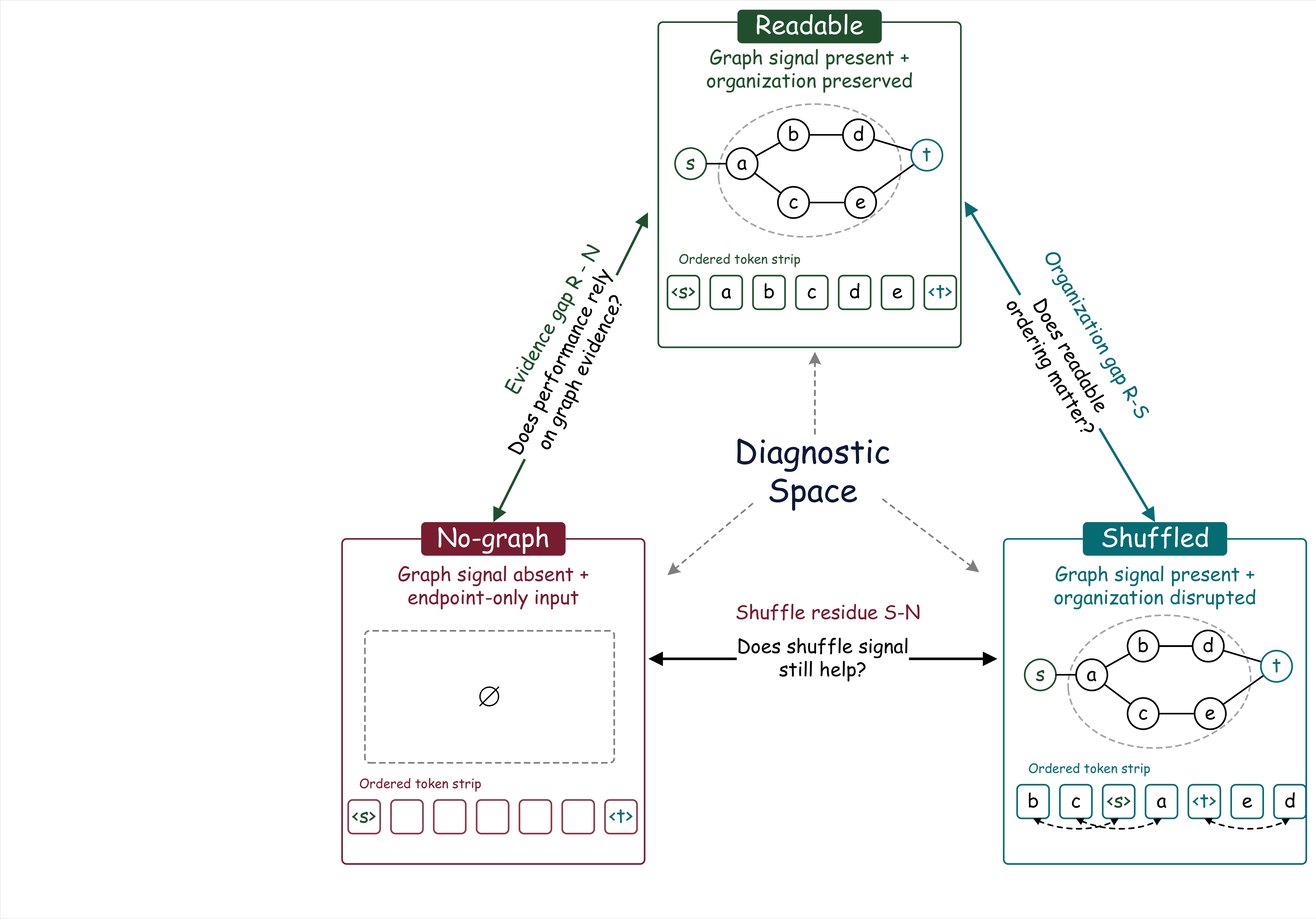}
    \vspace{-5pt}
    \caption{Intervention triangle for the readable, no-graph, and shuffled
        graph-token interfaces. The triangle shows how the three intervention
        conditions separate harmful-shuffle, shuffle-robust, and
        no-graph-saturated regimes.}
    \label{fig:domain-interventions}
    \vspace{-8pt}
\end{figure}

\begin{table}[t]
    \centering
    \small
    \setlength{\tabcolsep}{4pt}
    \begin{tabular}{@{\hspace{4pt}}lccc@{\hspace{4pt}}}
        \toprule
        Domain     & Readable & No graph & Shuffled \\
        \midrule
        DBLP       & $34.0$   & $15.0$   & $6.0$    \\
        Biomedical & $57.8$   & $57.1$   & $52.7$   \\
        GoodReads  & $76.6$   & $9.8$    & $34.2$   \\
        PubMed     & $52.0$   & $19.5$   & $51.3$   \\
        \bottomrule
    \end{tabular}
    \caption{Graph-token interventions. Values are StrictEM percentages from a single representative seed; Table~\ref{tab:native-generation} reports the multi-seed main results.}
    \label{tab:token-interventions}
\end{table}

Biomedical is no-graph-saturated: no-graph input already reaches $57.1$ StrictEM versus $57.8$ with readable tokens, exposing a strong task prior.

\subsection{Interface Ablations on Two Representative Domains}

This subsection tests Hypothesis 2: which interface conditions make graph evidence useful
for ordinary generation. The representative domains come from
Table~\ref{tab:token-interventions}: PubMed is shuffle-robust, and DBLP is
harmful-shuffle.

Table~\ref{tab:two-domain-ablation} reports the ablations. PubMed is most
sensitive to removing query-aware sampling, indicating evidence inclusion as
the limiting factor. DBLP is more sensitive to query syntax and degree cues,
consistent with Figure~\ref{fig:expr-summary}: broken organization is more
harmful in this domain. Together, the ablations support the order in
Section~\ref{sec:method}: include evidence, make it readable, and preserve
local structure.

\begin{table}[t]
    \centering
    \scriptsize
    \setlength{\tabcolsep}{3pt}
    \begin{tabular}{@{\hspace{4pt}}llcc@{\hspace{4pt}}}
        \toprule
        \multicolumn{2}{l}{Domain and selection pattern} & StrictEM $\uparrow$ & $\Delta$ vs. full \\
        \midrule
        \multicolumn{4}{l}{\textit{PubMed}: $N_\kappa\ll S_\kappa\approx R_\kappa$} \\
                                                         & full \method            & $52.0{\pm}3.4$ & -- \\
                                                         & no query-aware sampling & $25.0{\pm}1.1$ & $-27.0$ \\
                                                         & no role perception      & $47.0{\pm}2.3$ & $-5.0$ \\
                                                         & no distance/degree cues & $45.9{\pm}2.7$ & $-6.1$ \\
                                                         & no alignment            & $52.0{\pm}3.7$ & unchanged \\
        \midrule
        \multicolumn{4}{l}{\textit{DBLP}: $S_\kappa<N_\kappa<R_\kappa$} \\
                                                         & full \method            & $36.5{\pm}1.6$ & -- \\
                                                         & no query syntax         & $24.1{\pm}5.2$ & $-12.4$ \\
                                                         & no degree cues          & $22.8{\pm}3.7$ & $-13.7$ \\
                                                         & no role perception      & $32.0{\pm}2.3$ & $-4.5$ \\
                                                         & random order diagnostic & $33.8{\pm}5.0$ & $-2.7$ \\
        \bottomrule
    \end{tabular}
    \caption{Interface ablations on two representative domains. PubMed represents the shuffle-robust regime. DBLP represents the harmful-shuffle regime. Values are StrictEM percentages.}
    \label{tab:two-domain-ablation}
    \vspace{-6pt}
\end{table}

Removing the alignment loss changes PubMed little ($52.0\rightarrow52.0$), consistent with its shuffle-robust regime.

\subsection{HopQA Predictive Validation}

We further test whether \method and graph-only structure solve the same
examples. Table~\ref{tab:oracle-complementarity} reports a sample-level oracle
union on matched $1000$-example test predictions: an example is counted correct
when either \method or Pure GNN is correct. The oracle gains show that graph-only
structure and the graph-language interface solve overlapping but non-identical
subsets.

\begin{table}[H]
    \centering
    \scriptsize
    \setlength{\tabcolsep}{3pt}
    \begin{tabular}{@{}lccccc@{}}
        \toprule
        Domain & \method & Pure GNN & Oracle & \method-only & GNN-only \\
        \midrule
        DBLP       & $36.1$ & $31.3$ & $52.4$ & $21.1$ & $16.3$ \\
        Biomedical & $57.8$ & $27.9$ & $69.6$ & $41.7$ & $11.8$ \\
        GoodReads  & $76.6$ & $40.3$ & $87.2$ & $46.9$ & $10.5$ \\
        PubMed     & $52.0$ & $47.5$ & $72.8$ & $25.2$ & $20.8$ \\
        \bottomrule
    \end{tabular}
    \vspace{-4pt}
    \caption{Uses one matched seed for sample-level oracle analysis; Table~1 reports the headline control results and the multi-seed \method results.}
    \label{tab:oracle-complementarity}
    \vspace{-8pt}
\end{table}

Table~\ref{tab:zeroshot-transfer-main} gives a second validation view:
projected \method tokens preserve recoverable one-hop adjacency across
source-target domain pairs. Together, oracle complementarity and zero-shot
transfer indicate that \method is not merely replacing graph-only structure; it
makes a partially different portion of graph signal usable through the native
graph-language interface.

\begin{table}[H]
    \centering
    \scriptsize
    \setlength{\tabcolsep}{3pt}
    \begin{tabular}{@{}llcc@{}}
        \toprule
        Train & Test & Acc. & AUROC \\
        \midrule
        DBLP       & GoodReads  & $0.923{\pm}0.009$ & $0.958{\pm}0.002$ \\
        GoodReads  & DBLP       & $0.839{\pm}0.015$ & $0.925{\pm}0.005$ \\
        DBLP       & Biomedical & $0.811{\pm}0.021$ & $0.876{\pm}0.007$ \\
        Biomedical & GoodReads  & $0.892{\pm}0.017$ & $0.942{\pm}0.009$ \\
        PubMed     & DBLP       & $0.879{\pm}0.013$ & $0.921{\pm}0.013$ \\
        DBLP       & PubMed     & $0.868{\pm}0.001$ & $0.906{\pm}0.005$ \\
        \bottomrule
    \end{tabular}
    \vspace{-4pt}
    \caption{Zero-shot one-hop adjacency transfer. The probe is trained on source-domain pair labels only and uses no target-domain labels.}
    \label{tab:zeroshot-transfer-main}
    \vspace{-6pt}
\end{table}

\FloatBarrier

\section{Conclusion}
\label{sec:conclusion}
Five-choice HopQA exposes native usability failure in graph-language interfaces. Native strict evaluation and graph-token interventions separate graph-signal existence, interface readability, and native usability within one bounded diagnostic.

\method shows that a sampling-first readable interface improves native usability of exposed graph evidence, reaching native strict EM $36.5$--$76.6\%$ across four domains. Depending on the domain, unreadable graph tokens are harmful, residue-preserving, or no-graph-saturated.

\section*{Ethics Statement}
\label{sec:ethics}
This work uses public graph benchmark and baseline artifacts only for research evaluation; it involves no new personal data collection, user studies, or annotator exposure to offensive content. GRBench files, Planetoid/PubMed, and baseline artifacts are used according to their public releases or published descriptions, and our released code and processed diagnostics are for research use. Formal \method runs use LLaMA-3-8B-Instruct with bf16 and DeepSpeed on dual NVIDIA GeForce RTX 5090 GPUs; main experiments use a maximum of 12 epochs with early-stopping patience 6 and three seeds, path-witness diagnostics use the same maximum schedule, and the total compute is approximately 400 GPU-hours.

\section*{Limitations}
Hop-count QA is used as a controlled receiving-end probe for graph-language interfaces. Its valid single-integer answers make correctness, output validity, and fallback behavior directly measurable. Open-ended settings require additional protocols for comparing semantically equivalent answers.

The path-witness diagnostic is bounded and protocolized. Fully open-ended graph QA, path-witness generation, and KG reasoning require task-specific answer-equivalence and calibration protocols.

Numerical trends across decoder scales remain empirically untested; our
single-backbone results characterize LLaMA-3-8B-Instruct. Proposition~2 and
Corollary~2.1 characterize a fixed-interface information bound that does not
depend on decoder size.

\section*{Acknowledgements}
This work was supported by the Academy Member (Academician) Major Science and Technology Innovation Project under the ``Two Zones'' Science and Technology Development Program (Grant No. 2024LQ02001), the National Natural Science Foundation of China (Grant Nos. 62262064 and 62466057), the Key R\&D Project of Xinjiang Uygur Autonomous Region (Grant No. 2022295358), and the Xinjiang University Doctoral Postgraduate Innovation Project (Grant No. XJDX2025YJS084).

\bibliography{custom}

\appendix
\section{Notation Reference}
\label{app:notation}
Table~\ref{tab:notation-reference} collects the main symbols used across the paper. The table is intended as a quick reference for the task definition, interface conditions, and intervention diagnostics.

\begin{table*}[t]
\centering
\small
\setlength{\tabcolsep}{6pt}
\begin{tabular}{p{0.30\textwidth}p{0.60\textwidth}}
\toprule
Meaning & Symbol \\
\midrule
Graph & $G=(V,E,X)$ \\
Query & $q=(s,t)$ \\
Answer set & $\mathcal A=\{1,2,3,4,5\}$ \\
Gold answer & $y_i=d_{G_i}(s_i,t_i)$ \\
Model output & $\hat y_i$ \\
Evaluation set & $\mathcal D=\{(G_i,q_i,y_i)\}_{i=1}^{n}$ \\
Sampled subgraph & $G_{\mathrm{sub},i}=S_B(G_i,q_i)$ \\
Exposed decoder interface & $I_B(G_i,q_i)=I(G_{\mathrm{sub},i},q_i)=(T(q_i),M(q_i))$ \\
Ordinary generation path & $m_{\mathrm{nat}}$ \\
Graph-check set & $\mathcal C_{\mathrm{graph}}$ \\
Signal-use gap & $\Delta_{\mathrm{su}}^{+}(m)$ \\
Interface readability coordinates & $\Phi(I_B,q)=(C_{\mathrm{local}}(q),S_{\mathrm{role}}(q),A_{\mathrm{adj}}(q))$ \\
Output state & $O(\hat y_i)\in\mathcal O$ \\
Output-state distribution & $p_\theta(o\mid I)$ \\
Readable, no-graph, shuffled interfaces & $I_R,I_N,I_S$ \\
Domain & $\kappa$ \\
Readable, no-graph, shuffled exact match & $R_\kappa,N_\kappa,S_\kappa$ \\
Shuffled residue & $G^{\mathrm{shuf}}_\kappa=S_\kappa-N_\kappa$ \\
Readable organization gain & $L^{\mathrm{org}}_\kappa=R_\kappa-S_\kappa$ \\
\bottomrule
\end{tabular}
\caption{Unified notation reference used across the main text and appendix.}
\label{tab:notation-reference}
\end{table*}

\section{Datasets and Metrics}
\label{app:datasets-metrics}
We evaluate GRBench-style shortest-hop question answering on four graph
domains: DBLP, Biomedical, GoodReads, and PubMed. Each example asks for the hop
distance between a source node and a target node. The gold answer belongs to
the bounded answer set
\[
    \mathcal A=\{1,2,3,4,5\}.
\]
This bounded label space makes random guessing and majority-label prediction
directly comparable across domains.

\subsection{Datasets and Splits}

Table~\ref{tab:dataset-overview} summarizes the graph statistics used in the
controlled HopQA protocol. The processed edge counts correspond to the graph
files used for evaluation.

\paragraph{Dataset acquisition.}
DBLP, Biomedical, and GoodReads are obtained from the public GRBench release
\citep{jin2024graphcot}. Their domain QA files are available from the Hugging
Face dataset \texttt{PeterJinGo/GRBench}, and the corresponding
graph-environment files are distributed through the GRBench download link in
the Graph-CoT repository. PubMed is constructed from the public
Planetoid/PubMed citation graph \citep{yang2016planetoid}. We use these graph
files as source graphs and construct new balanced HopQA splits from
shortest-hop queries; the original GRBench QA labels are not used as HopQA
training or test labels.

\begin{table*}[t]
    \centering
    \small
    \setlength{\tabcolsep}{6pt}
    \begin{tabular}{lrrp{0.34\textwidth}}
        \toprule
        Domain     & Nodes      & Processed edges & Role in evaluation                                                     \\
        \midrule
        DBLP       & 11,453,104 & 87,551,311      & Large bibliographic graph with strong topology dependence.             \\
        Biomedical & 47,031     & 3,363,428       & Biomedical relation graph with dense processed connectivity.           \\
        GoodReads  & 3,784,086  & 11,704,748      & Recommendation-style graph with strong semantic and hub effects.       \\
        PubMed     & 19,717     & 88,648          & Citation-style graph used for both HopQA and path-witness diagnostics. \\
        \bottomrule
    \end{tabular}
    \caption{Graph statistics for the four controlled HopQA domains.}
    \label{tab:dataset-overview}
\end{table*}

Each formal domain uses matched train, validation, and test splits. The split
is balanced across the five hop labels. Therefore, both uniform random guessing
and fixed majority-label prediction have expected accuracy $0.200$.

\begin{table}[t]
    \centering
    \footnotesize
    \setlength{\tabcolsep}{3pt}
    \begin{tabular}{@{}lccc@{}}
        \toprule
        Split      & Per-hop examples & Total examples & Label prior \\
        \midrule
        Train      & $400 \times 5$   & 2000           & balanced    \\
        Validation & $40 \times 5$    & 200            & balanced    \\
        Test       & $200 \times 5$   & 1000           & balanced    \\
        \bottomrule
    \end{tabular}
    \caption{Balanced split construction for each domain. Each hop label in $\{1,2,3,4,5\}$ has the same number of examples.}
    \label{tab:balanced-splits}
\end{table}

The resulting chance and majority baselines are shown in
Table~\ref{tab:chance-baselines}. These baselines are used to compute
$\Delta$Chance in the main experimental tables.

\begin{table}[t]
    \centering
    \small
    \setlength{\tabcolsep}{6pt}
    \begin{tabular}{lcc}
        \toprule
        Baseline & Decision rule                   & Accuracy \\
        \midrule
        Random   & uniform guess over $\mathcal A$ & $20.0$   \\
        Majority & fixed majority label            & $20.0$   \\
        \bottomrule
    \end{tabular}
    \caption{Chance-level baselines under the balanced five-label HopQA protocol. Values are percentages.}
    \label{tab:chance-baselines}
\end{table}

\subsection{Metrics}

Table~\ref{tab:metric-definitions} defines the metrics used in the main
experiments. Strict exact match is the headline metric. Other metrics are audit
views of the generated output surface or auxiliary checks for graph-signal
recovery.

\begin{table*}[t]
    \centering
    \small
    \setlength{\tabcolsep}{4pt}
    \begin{tabular}{p{0.20\textwidth}p{0.34\textwidth}p{0.35\textwidth}}
        \toprule
        Metric                                                              & Definition                                                       & Use                  \\
        \midrule
        StrictEM                                                            &
        $\frac{1}{n}\sum_i \mathbf 1\{\operatorname{exact}(\hat y_i)=y_i\}$ &
        Main metric for ordinary generation. The output must be exactly one legal integer in $\mathcal A$.                                                            \\

        ParsedEM                                                            & $\frac{1}{n}\sum_i \mathbf 1\{\operatorname{extract}(\hat
        y_i)=y_i\}$                                                         & Output-surface audit using first-integer extraction when one exists, reported alongside StrictEM.                                                                                                                                    \\

        $\Delta$Chance                                                      &
        $\mathrm{StrictEM}-20.0$                                            &
        Difference from the balanced random/majority baseline.                                                                                                        \\

        SingleIntRate                                                       & $\frac{1}{n}\sum_i \mathbf 1\{\hat y_i \text{ is exactly one
        integer in } \mathcal A\}$                                          & Checks whether native generation produces a legal
        single-integer answer.                                                                                                                                        \\

        DominantAnswerRate                                                  & $\max_{a\in\mathcal A}\frac{1}{n}\sum_i\mathbf
        1\{\operatorname{extract}(\hat y_i)=a\}$                            & Measures collapse toward one
        dominant extracted answer. Lower is better.                                                                                                                   \\

        Path Found                                                          & fraction of examples where graph execution recovers a valid path &
        Audit metric for retrieval-execution diagnostics outside the LLM decoder.                                                                                     \\

        Answer Acc.                                                         & accuracy of the path-existence answer                            & Used in path-witness
        diagnostics.                                                                                                                                                  \\

        First-hop Joint Acc.                                                & fraction of examples with both correct answer and valid
        first-hop witness                                                   & Checks whether the predicted first hop is locally valid and
        moves toward the target.                                                                                                                                      \\ \bottomrule
    \end{tabular}
    \caption{Metric definitions. StrictEM is the primary metric. ParsedEM, SingleIntRate, DominantAnswerRate, Path Found, and path-witness metrics are used as diagnostic audits.}
    \label{tab:metric-definitions}
\end{table*}

For generated answers, $\operatorname{exact}(\hat y_i)$ returns a value only
when the output is exactly one valid integer in $\mathcal A$. The function
$\operatorname{extract}(\hat y_i)$ returns the first generated integer when one
exists. Outputs outside exact single-integer form are counted as StrictEM
failures, even when ParsedEM can recover a number.

\subsection{Leakage and Split Audits}

GoodReads receives an additional overlap audit because recommendation-style
graphs can contain hub shortcuts and repeated textual patterns.
Table~\ref{tab:goodreads-leakage-audit} reports the observed overlap checks
across train, validation, and test splits. All audited overlap types are zero.

\begin{table}[t]
    \centering
    \small
    \setlength{\tabcolsep}{5pt}
    \begin{tabular}{lc}
        \toprule
        Overlap type            & Observed overlap \\
        \midrule
        Qid overlap             & 0                \\
        Question text overlap   & 0                \\
        Question-answer overlap & 0                \\
        Endpoint-pair overlap   & 0                \\
        Exact-path overlap      & 0                \\
        Reverse-path overlap    & 0                \\
        \bottomrule
    \end{tabular}
    \caption{GoodReads leakage audit across train, validation, and test splits.}
    \label{tab:goodreads-leakage-audit}
\end{table}

We also inspect GoodReads performance by hop label as a sanity check against a
single-hop shortcut explanation. Table~\ref{tab:goodreads-hopwise} reports
hop-wise \method scores on the GoodReads test split.

\begin{table}[t]
    \centering
    \small
    \setlength{\tabcolsep}{6pt}
    \begin{tabular}{lccccc}
        \toprule
        Hop label     & 1     & 2     & 3     & 4     & 5     \\
        \midrule
        \method score & 0.608 & 0.877 & 0.547 & 0.857 & 0.943 \\
        \bottomrule
    \end{tabular}
    \caption{GoodReads hop-wise \method scores under the balanced HopQA test split. Values are proportions.}
    \label{tab:goodreads-hopwise}
\end{table}

\section{Experiment Details}
\label{app:experiment-details}

This appendix reports compact implementation details for the formal
experiments: run configuration, generation settings, controls, and run records.

\subsection{Run Configuration and Controls}

Table~\ref{tab:run-control-summary} consolidates the formal \method protocol.

\begin{table*}[t]
    \centering
    \footnotesize
    \setlength{\tabcolsep}{3pt}
    \begin{tabular}{p{0.14\textwidth}p{0.24\textwidth}p{0.50\textwidth}}
        \toprule
        Block & Setting & Value \\
        \midrule
        \multirow{4}{*}{Model}
        & LLM backbone & LLaMA-3-8B-Instruct; unfrozen during \method training; bf16 enabled; max new tokens 32. \\
        & Graph encoder & One-layer GAT-style encoder; hidden dimension 384; four attention heads; implementation name \texttt{vpa\_aegf\_gat\_light}. \\
        & Projector & Two-layer MLP projector maps graph encoder outputs into the LLM embedding space. \\
        & Sampling/tokens & Node-token mode with up to 32 projected nodes; source and target are forced into the seed set when available, with degree-ranked local expansion. \\
        \midrule
        \multirow{2}{*}{Structure}
        & Structural roles/order & Source, target, endpoint-proximity, and context roles; ordered by role, endpoint distances, degree, and traversal index. Gold labels and stored witness paths are not input features. \\
        & Alignment target & Projected node tokens are aligned against the self-looped degree-normalized sampled adjacency with fixed validation-selected $\lambda=0.25$. \\
        \midrule
        \multirow{2}{*}{Optim.}
        & Learning setup & Learning rate $8\times10^{-6}$; maximum 12 epochs; early-stopping patience 6; 2 warmup epochs; batch size 1; gradient accumulation 2; bf16; DeepSpeed; trainable LLM. \\
        & Seeds & $0/1/12138$. \\
        \midrule
        \multirow{4}{*}{Gen./ctrl}
        & Generation path & HuggingFace native generation path. \\
        & Primary scoring & strict exact match. \\
        & Audit scoring & first-integer extraction when one exists; invalid outputs count as StrictEM failures. \\
        & Non-main controls & SubgraphRAG retrieval-execution control, text-serialization control, Graph-CoT zero-shot control, and legal-choice scoring over $\{1,\ldots,5\}$. \\
        \bottomrule
    \end{tabular}
    \caption{Consolidated run, optimization, generation, and control settings for the formal \method protocol.}
    \label{tab:run-control-summary}
\end{table*}

The alignment loss and full training objective are defined in
Section~\ref{sec:training-objective-diagnostics}. This table records the
validated run settings used to instantiate that objective.

G-Retriever and LLaGA are evaluated using their released checkpoints and
method-specific training recipes; the 12-epoch S\textsuperscript{2}GE schedule
above applies only to our trainable runs.

\subsection{Baseline Output-State Audit}
\label{app:baseline-error-audit}

Table~\ref{tab:baseline-error-audit} summarizes representative native-generation
failure states for G-Retriever and LLaGA. The audit uses wrong generated
outputs from the preserved prediction run records and interprets ParsedEM as an
output-surface audit.

\begin{table*}[t]
    \centering
    \footnotesize
    \setlength{\tabcolsep}{4pt}
    \begin{tabular}{@{}p{0.12\textwidth}p{0.12\textwidth}p{0.24\textwidth}p{0.42\textwidth}@{}}
        \toprule
        Method & Domain pattern & Typical failure state & Representative output form \\
        \midrule
        G-Retriever
        & DBLP, GoodReads, PubMed
        & Numeric continuation after an answer-like digit
        & \texttt{3</s>4</s>3</s>4</s>...} or \texttt{1</s>1</s>1</s>...}; the first digit can match the gold label while the whole output is not a legal single integer. \\

        G-Retriever
        & Biomedical
        & Separator and repeated-digit continuation
        & \texttt{4</s> || 5</s>4</s>...} or \texttt{5</s>5</s>5</s>...}; extraction may recover a digit while native StrictEM remains zero. \\

        LLaGA
        & DBLP
        & Explanatory natural-language answer or no completed digit answer
        & \texttt{According to the graph, at least 3 people...} or an unfinished explanation with no legal answer token. \\

        LLaGA
        & Biomedical, PubMed
        & Prompt/dialogue continuation and answer-space escape
        & \texttt{5 USER: What is...} or text giving an invalid hop such as \texttt{8}; the output leaves the strict single-integer surface. \\

        LLaGA
        & GoodReads
        & Natural-language sentence containing an extractable digit
        & \texttt{The shortest hop count ... is 3.}; ParsedEM can count the digit, while ordinary generation still violates the requested output form. \\
        \bottomrule
    \end{tabular}
    \caption{Representative output-state failures for G-Retriever and LLaGA. These examples explain why ParsedEM is treated as an audit: it can extract digits from numeric continuations, explanatory prose, or dialogue continuations that are all StrictEM failures.}
    \label{tab:baseline-error-audit}
\end{table*}

\subsection{Control and Baseline Comparison}

Table~\ref{tab:baseline-comparison} records the auxiliary controls and
baseline variants. SubgraphRAG is the main-text name for the
retrieval-execution control; the appendix keeps the retrieval-execution mapping explicit.

\begin{table*}[t]
    \centering
    \footnotesize
    \setlength{\tabcolsep}{4pt}
    \begin{tabular}{@{}p{0.19\textwidth}p{0.16\textwidth}p{0.22\textwidth}p{0.35\textwidth}@{}}
        \toprule
        Control & Main-text name & Reported value & Concise role or failure label \\
        \midrule
        Retrieval-execution check & SubgraphRAG & Path Found: DBLP $24.9$, Biomedical $21.5$, GoodReads $58.2$, PubMed $50.6$ & Path Found measures executable-path coverage in the retrieved evidence graph; EM in Table~\ref{tab:native-generation} measures bounded hop-label accuracy. Biomedical and PubMed differ slightly because a found path can still map to a non-matching hop label. \\
        Text serialization baseline & raw-text control & DBLP $17.0/0.0/0.0\%$, Biomedical $16.6/0.0/0.0\%$, GoodReads $18.5/0.0/0.0\%$ & 128-token zero-shot adjacency lists; entries are permissive/strict/legal rates, with answer-1 dominance $83$--$89\%$. \\
        Graph-CoT zero-shot control & Graph-CoT & DBLP $0/1000$ valid completed answers & Responses halt before a valid \texttt{Finish[\ldots]} answer, so the task surface is mismatched. \\
        Legal-choice scoring & legal-choice control & DBLP R/N/S: G-Retriever $46.7/21.3/15.4$; LLaGA $74.1/38.6/27.2$; \method $88.6/62.0/45.1$ & Constrained scoring over $\mathcal A$. Favorable label-selection path for G-Retriever, LLaGA, and \method; the DBLP interface intervention is visualized in Figure~3(a). \\
        \bottomrule
    \end{tabular}
    \caption{Compact comparison of auxiliary baselines and controls. SubgraphRAG Path Found reports executable-path coverage, while SubgraphRAG EM in Table~\ref{tab:native-generation} reports bounded hop-label accuracy; the two can differ when a recovered path maps to a non-matching label.}
    \label{tab:baseline-comparison}
\end{table*}

\subsection{Run Records}

For reproducibility, each run record stores the domain, seed, prediction file
path, evaluation path, and ordered test-ID checksum. These records are used to
verify that seed-level reports and matched diagnostic comparisons are computed
over the intended test order.

\begin{table}[t]
    \centering
    \footnotesize
    \setlength{\tabcolsep}{3pt}
    \begin{tabular}{@{}p{0.34\columnwidth}p{0.54\columnwidth}@{}}
        \toprule
        Record field             & Purpose                                                                                        \\
        \midrule
        Domain                   & identifies DBLP, Biomedical, GoodReads, or PubMed                                              \\
        Seed                     & identifies the run seed, such as $0/1/12138$                                                   \\
        Prediction path          & locates the generated output or diagnostic prediction file                                     \\
        Evaluation path          & identifies native generation, graph-only control, retrieval execution, or legal-choice control \\
        Ordered test-ID checksum & verifies that compared predictions use the same test order                                     \\
        \bottomrule
    \end{tabular}
    \caption{Run fields recorded for reproducibility and matched evaluation.}
    \label{tab:run-records}
\end{table}

\section{Diagnostic Protocol}
\label{sec:diagnostics}
\noindent\textbf{Diagnostic target.}
The method section defines the graph-language interface. This appendix defines the observable quantities used to audit whether that interface remains readable during training and evaluation.

\subsection{Readability}

Interface readability means that the projected token stream preserves enough
local structure to recover endpoint roles, local adjacency, and
endpoint-conditioned path order. Downstream EM depends on readability and also
reflects decoder optimization, so we audit the interface with a frozen-probe
protocol in addition to the main task.

The within-domain probe samples positive and negative one-hop node pairs from
local subgraphs and trains a lightweight binary head on frozen pair features.
The cross-domain probe trains the same head on one source domain and evaluates
it on another target domain using frozen target representations. Table~\ref{tab:zeroshot-transfer-main}
reports the six-direction zero-shot matrix in the main text; this probe isolates
adjacency recoverability from hop-decoder adaptation.

\subsection{Path-Witness Protocol}

The path-witness diagnostic keeps the task bounded while requiring more than a
scalar hop label. Each example asks for a protocolized path-existence answer
and a witness field that is checked by local adjacency and endpoint progress.
It is a task-surface diagnostic: the model must recover topology evidence in a
structured output, while evaluation stays bounded.

The graph-only path-witness control uses the same path-witness protocol. It is
not the five-way hop classifier reported in Table~\ref{tab:native-generation}.
The main PubMed path-witness results are reported in Table~\ref{tab:path-witness}.

\subsection{Graph-Token Intervention Triangle}

Readable, no-graph, and shuffled interfaces form the shared intervention
coordinates used to separate domain regimes. The main text visualizes these
coordinates in Figure~\ref{fig:domain-interventions}. The two contrasts are
shuffled residue $G^{\mathrm{shuf}}_\kappa=S_\kappa-N_\kappa$ and readable
organization gain $L^{\mathrm{org}}_\kappa=R_\kappa-S_\kappa$.

\subsection{Output Variables}
\label{app:diagnostic-output-vars}

We summarize checkpoint behavior with
\[
    \macrostate=(e,\ell,d,s,c).
\]
Here $e$ is exact match, $\ell$ is legal-output rate, $d$ is long-digit
continuation rate, $s$ is single-integer output rate, and $c$ is
dominant-answer concentration. DBLP uses \digitprobe as the most stable failure
observable. It counts generations with an ID-like numeric continuation of
length at least six, dominated by repeated 8s. GoodReads and PubMed emphasize
dominant-answer concentration. Biomedical additionally requires legality and
format-collapse monitoring.

The DBLP checkpoint-PCA value ($98.52\%$) is the first data-fitted component of
this macro-state sequence; the prototype-axis value ($98.52\%$) is variance
explained by a fixed bad-to-good axis.

\subsection{Probe Selection}

Probe discovery follows a fixed workflow.
\begin{enumerate}[leftmargin=1.2em]
    \item collect raw generations from validation or intervention sweeps;
    \item measure extracted-integer frequency, repeated fragments, empty outputs, and
          dominant-answer concentration;
    \item compare those quantities against the gold-label distribution;
    \item promote the most stable anomaly to a domain probe;
    \item track that probe jointly with EM and legality across checkpoints.
\end{enumerate}
The domain-specific probes follow this procedure.

\section{Supplementary Diagnostics}
\label{app:supp-diagnostics}
\paragraph{Role Implementation.}
\label{app:role-implementation}
Role ids mark source, target, local endpoint-proximity, and context inside the sampled subgraph. Source and target are query inputs. Endpoint proximity uses only BFS distances and connectivity within the bounded sample. Full-graph shortest paths, hop labels, and stored witnesses are evaluation-only objects.

\subsection{Formal Boundary Results}
\label{app:formal-boundary}

This appendix formalizes the boundedness assumptions behind the diagnostic. The three propositions show that bounded HopQA queries are graph-structure functions, bounded exposed interfaces impose an information boundary, and bounded output-state audits make interface changes measurable. The hypotheses are tested empirically in Section~\ref{sec:exp}.

\begin{proposition}[Bounded HopQA queries are graph-structure functions]
    Let $G=(V,E,X)$ be a simple loopless unweighted graph, and let $q=(s,t)$. The HopQA label
    \[
        y(G,q)=d_G(s,t)
    \]
    is determined by the adjacency matrix $A_G$. Moreover, if
    $D_G=(d_G(u,v))_{u,v\in V}$ is the all-pair hop distance matrix, then $D_G$
    determines $A_G$. For $u\ne v$,
    \[
        (A_G)_{uv}=1
        \quad\Longleftrightarrow\quad
        d_G(u,v)=1 .
    \]
\end{proposition}

\begin{proof}[\textbf{Proof 1}]
    Given $A_G$, the edge set $E$ is fixed. Hence the set of all paths from $s$ to $t$ is fixed, and therefore
    \[
        \begin{aligned}
            d_G(s,t)=\min\{k:\  & \exists v_0,\ldots,v_k, \\
                                & v_0=s,\ v_k=t,          \\
                                & (v_{i-1},v_i)\in E,     \\
                                & i=1,\ldots,k\}.
        \end{aligned}
    \]
    This quantity is fixed.

    For $u\ne v$, since $G$ is simple, loopless, and unweighted,
    \[
        (u,v)\in E
        \quad\Longleftrightarrow\quad
        d_G(u,v)=1 .
    \]
    By the definition of adjacency matrix,
    \[
        (A_G)_{uv}=1
        \quad\Longleftrightarrow\quad
        (u,v)\in E .
    \]
    Thus
    \[
        (A_G)_{uv}=1
        \quad\Longleftrightarrow\quad
        d_G(u,v)=1 ,
        \qquad u\ne v .
    \]
    Since $G$ is loopless,
    \[
        (A_G)_{uu}=0 .
    \]
    Therefore $D_G$ determines $A_G$.
\end{proof}

\begin{proposition}[Bounded exposed interfaces impose an information boundary]
    Let $\mathcal X$ be a set of labeled examples, with label $y(x)\in\mathcal Y$. Let
    \[
        I:\mathcal X\to\mathcal Z
    \]
    be the exposed interface. A decoder observes $x$ only through $I(x)$, so its
    output distribution is
    \[
        h(\cdot\mid I(x))\in\Delta(\mathcal Y).
    \]
    For any $z\in\mathcal Z$, define
    \[
        \mathcal X_z=\{x\in\mathcal X:I(x)=z\}.
    \]
    Then
    \[
        \begin{aligned}
             &
            \max_h
            \frac{1}{|\mathcal X_z|}
            \sum_{x\in\mathcal X_z}
            h(y(x)\mid z)
            \\
             & \quad =
            \max_{a\in\mathcal Y}
            \frac{|\{x\in\mathcal X_z:y(x)=a\}|}{|\mathcal X_z|}.
        \end{aligned}
    \]
    In particular, if $\mathcal X_z$ contains two examples with different labels,
    no decoder observing only $z$ can be correct on both.
\end{proposition}

\begin{proof}[\textbf{Proof 2}]
    Fix $z$. For all $x\in\mathcal X_z$,
    \[
        I(x)=z .
    \]
    Thus the decoder uses the same distribution $h(\cdot\mid z)$. Let
    \[
        n_a=|\{x\in\mathcal X_z:y(x)=a\}|.
    \]
    The average correctness is
    \[
        \frac{1}{|\mathcal X_z|}
        \sum_{x\in\mathcal X_z}
        h(y(x)\mid z)
        =
        \sum_{a\in\mathcal Y}
        \frac{n_a}{|\mathcal X_z|}
        h(a\mid z).
    \]
    Since $h(\cdot\mid z)$ is a probability distribution,
    \[
        h(a\mid z)\ge 0,
        \qquad
        \sum_{a\in\mathcal Y}h(a\mid z)=1 .
    \]
    Therefore
    \[
        \begin{aligned}
            \sum_{a\in\mathcal Y}
            \frac{n_a}{|\mathcal X_z|}
            h(a\mid z)
            \le &
            \left(\max_{a\in\mathcal Y}
            \frac{n_a}{|\mathcal X_z|}\right)
            \sum_{a\in\mathcal Y}h(a\mid z)
            \\
            =   &
            \max_{a\in\mathcal Y}
            \frac{n_a}{|\mathcal X_z|}.
        \end{aligned}
    \]
    The bound is attained by choosing any
    \[
        a^\star\in\arg\max_{a\in\mathcal Y} n_a
    \]
    and setting
    \[
        h(a^\star\mid z)=1 .
    \]
    Hence equality holds.

    If $x_1,x_2\in\mathcal X_z$ and
    \[
        y(x_1)\ne y(x_2),
    \]
    let
    \[
        y_1=y(x_1),\qquad y_2=y(x_2).
    \]
    Since $y_1\ne y_2$,
    \[
        h(y_1\mid z)+h(y_2\mid z)\le 1 .
    \]
    Thus no single decoder distribution conditioned only on $z$ can be correct on
    both examples.
\end{proof}

\paragraph{Corollary 2.1 (Decoder scaling cannot remove interface-induced ambiguity).}
Under a fixed exposed interface, enlarging the decoder family may improve
approximation, optimization, and output formation, but cannot recover
distinctions that have already been removed by that interface.

\begin{proposition}[Bounded output-state audits make interface changes measurable]
    Let $\mathcal O$ be a finite set of output states, and let
    \[
        O:\hat{\mathcal Y}\to\mathcal O
    \]
    map a generated output to its state. For an interface condition $I$, define
    \[
        p_\theta(o\mid I)
        =
        \mathbb P_\theta(O(\hat y)=o\mid I).
    \]
    For two interface conditions $I_1,I_2$, define
    \[
        D_{\mathcal O}(I_1,I_2)
        =
        \frac12
        \sum_{o\in\mathcal O}
        \left|
        p_\theta(o\mid I_1)-p_\theta(o\mid I_2)
        \right|.
    \]
    Then
    \[
        D_{\mathcal O}(I_1,I_2)\ge 0,
    \]
    and
    \[
        \begin{aligned}
            D_{\mathcal O}(I_1,I_2)=0
            \quad\Longleftrightarrow\quad
             & p_\theta(o\mid I_1)                             \\
             & =p_\theta(o\mid I_2),\ \forall o\in\mathcal O .
        \end{aligned}
    \]
    Consequently, if there exists $o^\star\in\mathcal O$ such that
    \[
        p_\theta(o^\star\mid I_1)\ne p_\theta(o^\star\mid I_2),
    \]
    then
    \[
        D_{\mathcal O}(I_1,I_2)>0 .
    \]
\end{proposition}

\begin{proof}[\textbf{Proof 3}]
    For every $o\in\mathcal O$,
    \[
        \left|
        p_\theta(o\mid I_1)-p_\theta(o\mid I_2)
        \right|\ge 0 .
    \]
    Hence
    \[
        D_{\mathcal O}(I_1,I_2)\ge 0 .
    \]

    If
    \[
        D_{\mathcal O}(I_1,I_2)=0,
    \]
    then
    \[
        \sum_{o\in\mathcal O}
        \left|
        p_\theta(o\mid I_1)-p_\theta(o\mid I_2)
        \right|
        =0 .
    \]
    Each term is nonnegative, so each term is zero:
    \[
        \left|
        p_\theta(o\mid I_1)-p_\theta(o\mid I_2)
        \right|=0,
        \qquad \forall o\in\mathcal O .
    \]
    Thus
    \[
        p_\theta(o\mid I_1)=p_\theta(o\mid I_2),
        \qquad \forall o\in\mathcal O .
    \]

    Conversely, if
    \[
        p_\theta(o\mid I_1)=p_\theta(o\mid I_2),
        \qquad \forall o\in\mathcal O,
    \]
    then every absolute value term is zero, so
    \[
        D_{\mathcal O}(I_1,I_2)=0 .
    \]

    Finally, if there exists $o^\star\in\mathcal O$ such that
    \[
        p_\theta(o^\star\mid I_1)\ne p_\theta(o^\star\mid I_2),
    \]
    then
    \[
        \left|
        p_\theta(o^\star\mid I_1)-p_\theta(o^\star\mid I_2)
        \right|>0 .
    \]
    Thus the sum contains a positive term, and
    \[
        D_{\mathcal O}(I_1,I_2)>0 .
    \]
\end{proof}

\FloatBarrier
\subsection{Protocols}
\label{app:appendix-diagnostic-protocols}
\paragraph{Sampling and Alignment.}
The sampling audit operationalizes evidence inclusion by endpoint coverage and
path recall. On 100 DBLP hop-QA examples, degree-only sampling gives endpoint
coverage/path recall $0.00/0.00$, query-aware 1-hop gives $1.00/0.58$, and
query-aware 2-hop gives $1.00/0.835$. The alignment objective is defined in
Section~\ref{sec:training-objective-diagnostics}; the formal protocol uses the
validation-selected $\lambda=0.25$ reported in Table~\ref{tab:run-control-summary}.

\paragraph{Ablation Protocol.}
Table~\ref{tab:two-domain-ablation} reports the main-text endpoint-regime
ablations. PubMed tests the shuffle-robust interface-coordinate hierarchy. DBLP
tests the harmful-shuffle role/query readability stress case.

\paragraph{Adjacency Probes.}
Zero-shot adjacency probes test $A_{\mathrm{adj}}$ independently of hop-answer
decoding. On DBLP reachability, \method gives AUROC
$(0.950,0.865,0.861,0.857,0.873)$ for $k=1,\dots,5$, while Pure GNN gives
$(0.818,0.952,0.802,0.722,0.720)$. The six-direction cross-domain transfer
results are reported once in Table~\ref{tab:zeroshot-transfer-main}.

\end{document}